\documentclass{colt2015} 
\usepackage{hyperref}
\usepackage{url}
\usepackage{times}
\usepackage[algo2e]{algorithm2e}

\usepackage{bbm}
\usepackage{graphics, graphicx, xcolor}
\usepackage{enumitem}
\usepackage{verbatim}		
\usepackage{stmaryrd}
\usepackage{float}
\usepackage[mathscr]{euscript}

\newtheorem{thm}{Theorem}
\newtheorem{lem}[thm]{Lemma}
\newtheorem{prop}[thm]{Proposition}
\newtheorem{cor}[thm]{Corollary}

\newtheorem{obs}[thm]{Observation}
\newtheorem{defn}[thm]{Definition}

\newcommand{\corr}{\mbox{corr}}

\newcommand{\vF}{\mathbf{F}} 
  
\newcommand{\vh}{\mathbf{h}}
\newcommand{\vx}{\mathbf{x}}
\newcommand{\vb}{\mathbf{b}} 
\newcommand{\vu}{\mathbf{u}}   
\newcommand{\vl}{\mathbf{l}}
    
\newcommand{\vg}{\mathbf{g}}

\newcommand{\vz}{\mathbf{z}}
\newcommand{\valpha}{\vec{\alpha}}

\newcommand{\vzero}{\mathbf{0}}
\newcommand{\vone}{\mathbf{1}}

\DeclareMathOperator*{\argmin}{arg\,min}
\DeclareMathOperator*{\argmax}{arg\,max}
\DeclareMathOperator{\sgn}{sgn}
\DeclareMathOperator{\Prtxt}{Pr}

\newcommand{\RR}{\mathbb{R}}      

\newcommand{\vnorm}[1]{\left\lVert#1\right\rVert} 
\newcommand{\evp}[2]{\mathbb{E}_{#2} \left[#1\right]} 
\newcommand{\abs}[1]{\left| #1 \right|}
\newcommand{\pr}[1]{\Prtxt \left(#1\right)}

\newcommand{\zbr}{\textsc{ZBR}}
\newcommand{\wmv}{\textsc{WMV}}

\newcommand{\sighat}{\hat{\sigma}}

\newcommand{\cH}{\mathcal{H}}
\newcommand{\cX}{\mathcal{X}}

\newcommand{\cD}{\mathcal{D}}

\newcommand{\lrp}[1]{\left(#1\right)}
\newcommand{\lrb}[1]{\left[#1\right]}

\makeatletter
\newcommand{\thickhline}{%
    \noalign {\ifnum 0=`}\fi \hrule height 1pt
    \futurelet \reserved@a \@xhline
}
\makeatother

\begin{document}

\title[Minimax Classifier Aggregation]{Optimally Combining Classifiers Using Unlabeled Data}

\coltauthor{\Name{Akshay Balsubramani} \Email{abalsubr@ucsd.edu} \\
\Name{Yoav Freund} \Email{yfreund@ucsd.edu}\\
\addr Dept. of CSE, University of California San Diego, La Jolla, CA 92093, USA
}

\maketitle

\begin{abstract}

We develop a worst-case analysis of aggregation of classifier ensembles for binary classification. 
The task of predicting to minimize error is formulated as a game played 
over a given set of unlabeled data (a transductive setting), 
where prior label information is encoded as constraints on the game. 
The minimax solution of this game identifies cases where a
weighted combination of the classifiers can perform significantly
better than any single classifier.

\end{abstract}
\begin{keywords}
Ensemble aggregation, transductive, minimax
\end{keywords}


\section{Introduction}
Suppose that we have a finite set, or ensemble, of binary classifiers
$\cH = \{ h_1,h_2,\ldots,h_p \}$, 
with each $h_i$ mapping data in some space $\cX$ to a binary prediction $\{-1,+1\}$. 
Examples $(x,y) \in \cX \times \{-1,+1\}$ 
are generated i.i.d. according to some fixed but unknown distribution $\cD$, 
where $y \in \{-1,+1\}$ is the class \emph{label} of the example. 
We write the expectation with respect to $\cD$ or one of its marginal distributions as $\evp{\cdot}{\cD}$.

Consider a statistical learning setting, 
in which we assume access to two types of i.i.d. data: a small set of labeled training examples  
$S = \{(x'_1,y'_1),\ldots(x'_m,y'_m)\}$ drawn from $\cD$ and a much larger
set of unlabeled test examples $U = \{x_1,\ldots,x_n\}$ drawn i.i.d. 
according to the marginal distribution over $\cX$ induced by $\cD$. 
A typical use of the labeled set is to find an upper bound on the expected error rate of each of the classifiers in the ensemble. 
Specifically, we assume a set of lower bounds $\{b_i\}_{i=1}^p > 0$ such 
that the correlation $\corr(h_i) := \evp{y h_i (x)}{\cD}$ satisfies $\corr(h_i) \geq b_i$.

If we ignore the test set, then the best we can do, in the worst case,
is to use the classifier with the largest correlation (smallest error).
This corresponds to the common practice of {\em empirical risk minimization} (ERM). 
However, in many cases we can glean useful
information from the distribution of the test set that will allow us
to greatly improve over ERM.

We motivate this statement by contrasting two simple prediction scenarios, A and B. 
In both cases there are $p=3$ classifiers and $n=3$ unlabeled test
examples. The correlation vector is $\vb = (1/3,1/3,1/3)$; equivalently,
the classifier error rates are $33\%$. 
Based on that information, 
the predictor knows that each classifier makes two correct predictions and one
incorrect prediction.

So far, both cases are the same. The difference is in the relations
between different predictions on the same example.  In case A, each
example has two predictions that are the same, and a third that is
different.  In this case it is apparent that the majority vote
over the three classifiers has to be correct on all 3 examples,
i.e. we can reduce the error from $\frac{1}{3}$ to $0$. In case B, all
three predictions are equal for all examples. In other words, the
three classification rules are exactly the same on the three examples, 
so there is no way to improve over any single rule.

These cases show that there is information in the unlabeled test
examples that can be used to reduce error -- indeed,  
cases A and B can be distinguished without using any labeled examples.
In this paper, we give a complete characterization of the optimal worst-case (minimax) predictions given the
correlation vector $\vb$ and the unlabeled test examples.

Our development does not consider the feature space $\cX$ directly, 
but instead models the knowledge of the $p$ ensemble predictions on $U$ 
with a $p \times n$ matrix that we denote by $\vF$. Our focus is on
how to use the matrix $\vF$ in conjunction with the correlation
vector $\vb$, to make minimax optimal predictions on the test examples.  

The rest of the paper is organized as follows. 
In Section~\ref{sec:setup} we introduce some additional notation. 
In Section~\ref{sec:game1} we formalize the above intuition as a zero-sum game between the predictor and
an adversary, and solve it, characterizing the minimax strategies for both
sides by minimizing a convex objective we call the slack function. 
This solution is then linked to a statistical learning algorithm in Section \ref{sec:uniform-convergence}. 

In Section~\ref{sec:discthmpred}, we interpret the slack
function and the minimax strategies, providing more toy examples following the one given above to build intuition. 
In Section~\ref{sec:alg} we focus on computational issues in running the statistical learning algorithm. 
After discussing relations to other work in Section~\ref{sec:relwork}, we conclude in Section~\ref{sec:openproblems}.

\section{Preliminaries}
\label{sec:setup}

The main tools we use in this paper are linear programming and uniform
convergence. We therefore use a combination of matrix notation and
the probabilistic notation given in the introduction. 
The algorithm is first described in a deterministic context where some inequalities are assumed to hold; probabilistic
arguments are used to show that these assumptions are correct with high probability.

The ensemble's predictions on the unlabeled data are denoted by $\vF$:
\begin{equation}
\vF = 
 \begin{pmatrix}
   h_1(x_1) & h_1(x_2) & \cdots & h_1 (x_n) \\
   \vdots   & \vdots    & \ddots &  \vdots  \\
   h_p(x_1)  &  h_p (x_2)  & \cdots &  h_p (x_n)
 \end{pmatrix}
 \in [-1, 1]^{p \times n}
\end{equation}
The \textbf{true labels} on the test data $U$ are represented by $\vz
= (z_1; \dots; z_n) \in [-1,1]^n$.

Note that we allow $\vF$ and $\vz$ to
take any value in the range $[-1, 1]$ rather than just the two
endpoints. This relaxation does not change the analysis, because intermediate
values can be interpreted as the expected value of randomized
predictions.  For example, a value of $\frac{1}{2}$ indicates $\{+1\;
\text{w.p.}\; \frac{3}{4} $, $-1\; \text{w.p.}\; \frac{1}{4} \}$. This
interpretation extends to our definition of the correlation
on the test set,
$\widehat{\corr}_{U} (h_i) = \frac{1}{n} \sum_{j=1}^n h_i (x_j) z_j$.~\footnote{We are slightly abusing the
  term ``correlation'' here. Strictly speaking this is just the expected
  value of the product, without standardizing by mean-centering and rescaling for unit variance. 
  We prefer this to inventing a new term.}

The labels $\vz$ are hidden from the predictor, 
but we assume the predictor has knowledge of a {\bf correlation vector}
$\vb \geq \vzero^n$ such that $\widehat{\corr}_{U} (h_i) \geq b_i$ for all $i \in [p]$, 
i.e. $ \frac{1}{n} \vF \vz \geq \vb$. 
From our development so far, the correlation vector's components $b_i$ each correspond 
to a constraint on the corresponding classifier's test error $\frac{1}{2} (1 - b_i)$. 

The following notation is used throughout the paper: $[a]_{+} = \max (0, a)$ and $[a]_{-} = [-a]_{+}$,  
$[n] = \{ 1,2,\dots,n \}$, $\vone^n = (1; 1; \dots; 1) \in \RR^n$, and $\vzero^n$
similarly.  Also, write $I_n$ as the $n \times n$ identity matrix.
All vector inequalities are componentwise. 
The probability simplex in $d$ dimensions is denoted by $\Delta^d = \{ \sigma \geq \vzero^d : \sum_{i=1}^d \sigma_i = 1 \}$.
Finally, we use vector notation for the rows and columns of $\vF$: 
$\vh_i = (h_i(x_1), h_i(x_2), \cdots, h_i (x_n))^\top$ and $\vx_j =
(h_1(x_j), h_2(x_j), \cdots, h_p (x_j))^\top$.


\section{The Transductive Binary Classification Game}
\label{sec:game1}

We now describe our prediction problem, and formulate it as a zero-sum game between 
two players: a predictor and an adversary.

In this game, the predictor is the first player, 
who plays $\vg = (g_1; g_2; \dots; g_n)$, 
a randomized label $g_i \in [-1,1]$ for each example $\{\vx_i\}_{i=1}^{n}$. 
The adversary then plays, setting the labels $\vz \in [-1,1]^n$ 
under ensemble test error constraints defined by $\vb$. 
The predictor's goal is to minimize (and the adversary's to maximize) 
the \emph{worst-case expected classification error on the test data} 
(w.r.t. the randomized labelings $\vz$ and $\vg$): 
$\frac{1}{2} \lrp{1 - \frac{1}{n} \vz^\top \vg }$. 
This is equivalently viewed as maximizing worst-case correlation $\frac{1}{n} \vz^\top \vg $. 

To summarize concretely, we study the following game:
\begin{align}
\label{game1eq}
\displaystyle 
V := \max_{\vg \in [-1,1]^n} \; \min_{\substack{ \vz \in [-1,1]^n , \\ \frac{1}{n} \vF \vz \geq \vb }} \;\; \frac{1}{n} \vz^\top \vg
\end{align}

It is important to note that we are only modeling ``test-time''
prediction, and represent the information gleaned from the labeled data
by the parameter $\vb$. Inferring the vector $\vb$ from training data
is a standard application of Occam's Razor \cite{BEHW87}, which we provide in
Section~\ref{sec:uniform-convergence}.

The minimax theorem (e.g. \cite{CBL06}, Theorem 7.1) applies to the game \eqref{game1eq}, 
since the constraint sets are convex and compact and the payoff linear. 
Therefore, it has a minimax equilibrium and associated optimal 
strategies $\vg^*, \vz^*$ for the two sides of the game, i.e. 
$\min_{\vz}\; \frac{1}{n} \vz^\top \vg^* = V = \max_{\vg} \frac{1}{n} \vz^{*^\top} \vg$ .

As we will show, both optimal strategies are simple functions of a
particular \emph{weighting} over the $p$ hypotheses -- a nonnegative $p$-vector. 
Define this weighting as follows.
\begin{defn}[\textbf{Slack Function and Optimal Weighting}]
Let $\sigma \geq 0^p$ be a weight vector over $\cH$ (not necessarily a distribution).
The vector of \textbf{ensemble predictions} is
$\vF^\top \sigma = (\vx_1^\top \sigma, \dots, \vx_n^\top \sigma)$, 
whose elements' magnitudes are the \textbf{margins}. 
The \textbf{prediction slack function} is
\begin{align}
\label{eqn:slack}
\gamma (\sigma, \vb) = \gamma (\sigma) := \frac{1}{n} \sum_{j=1}^n \left[ \abs{\vx_{j}^\top \sigma} - 1 \right]_{+} - \vb^\top \sigma
\end{align}
An \textbf{optimal weight vector} $\sigma^*$ is any minimizer of the slack function: 
$\displaystyle \sigma^* \in \argmin_{\sigma \geq 0^p} \left[ \gamma (\sigma) \right]$.
\end{defn}

Our main result uses these to describe the solution of the game \eqref{game1eq}.
\begin{thm}[Minimax Equilibrium of the Game]
\label{thm:gamesolngen}
The minimax value of the game \eqref{game1eq} is 
$V = - \gamma (\sigma^*)$. 
The minimax optimal strategies are defined as follows:
for all $i \in [n]$,
\begin{align}
g_i^* \doteq g_i (\sigma^*) = \begin{cases} \vx_{i}^\top \sigma^* & \abs{\vx_{i}^\top \sigma^*} < 1 \\ 
\sgn(\vx_{i}^\top \sigma^*) & \mbox{otherwise} \end{cases}
\quad \quad \text{and} \quad \quad
z_i^* = 
\begin{cases} 
0 & \abs{\vx_{i}^\top \sigma^*} < 1 \\ 
\sgn(\vx_{i}^\top \sigma^*) & \abs{\vx_{i}^\top \sigma^*} > 1 
\end{cases}
\label{eqn:opt-strats}
\end{align}
\end{thm}

\begin{figure}
\centering
\includegraphics[width=0.55\textwidth]{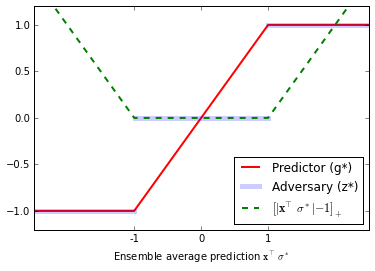}
\caption{
\label{fig:optstrats}
The optimal strategies and slack function as a function of the ensemble prediction $\vx^\top \sigma^*$.}
\end{figure}

The proof of this theorem is a standard application of Lagrange duality and the minimax theorem.
The minimax value of the game and the optimal strategy for the predictor $\vg^*$ (Lemma~\ref{lem:game1gopt}) 
are our main objects of study and are completely characterized, 
and the theorem's partial description of $\vz^*$ (proved in Lemma \ref{lem:game1zopt}) 
will suffice for our purposes. 
\footnote{For completeness, Corollary \ref{cor:zoptfullpred} in the appendices 
specifies $z_i^*$ when $\abs{\vx_{i}^\top \sigma^*} = 1$.}

Theorem \ref{thm:gamesolngen} illuminates the importance of the optimal weighting $\sigma^*$ over hypotheses. 
This weighting $\sigma^* \in \argmin_{\sigma \geq 0^p} \gamma (\sigma)$ is the solution 
to a convex optimization problem (Lemma \ref{lem:helperthreshcvx}), 
and therefore we can efficiently compute it and $\vg^*$ to any desired accuracy. 
The ensemble prediction (w.r.t. this weighting) on the test set is $\vF^\top \sigma^*$, 
which is the only dependence of the solution on $\vF$. 

More specifically, the minimax optimal prediction and label \eqref{eqn:opt-strats} on any test set example $\vx_j$ 
can be expressed as functions of the ensemble prediction $\vx_j^\top \sigma^*$ 
on that test point alone, without considering the others. 
The $\vF$-dependent part of the slack function also depends separately on each test point's ensemble prediction. 
Figure \ref{fig:optstrats} depicts these three functions.

\section{Bounding the Correlation Vector}
\label{sec:uniform-convergence}

In the analysis presented above we assumed that a correlation vector $\vb$ is given, 
and that each component is guaranteed to be a lower bound on the test correlation 
of the corresponding hypothesis. 
In this section, we show how $\vb$ can be calculated from a labeled training set.

The algorithm that we use is a natural one using uniform convergence: we compute the empirical
correlations for each of the $p$ classifiers, and add a uniform penalty term to guarantee
that the $\vb$ is a lower bound on the correlation of the test data.
For each classifier, we consider three quantities:
\begin{itemize}
\item The true correlation: $\corr(h) = \evp{y h(x)}{\cD}$
\item The correlation on the training set of labeled data: $\widehat{\corr}_{S}(h) =
  \frac{1}{m} \sum_{i=1}^m h(x'_i) y'_i$
\item The correlation on the test set of unlabeled data: $\widehat{\corr}_{U}(h) =
  \frac{1}{n} \sum_{i=1}^n h(x_i) z_i$
\end{itemize}
Using Chernoff bounds, we can show that the training and test
correlations are concentrated near the true correlation. 
Specifically, for each individual classifier $h$ we have the two inequalities 
\begin{eqnarray*}
\pr{\widehat{\corr}_{S} (h) > \corr(h) + \epsilon_{S}} \leq e^{-2m \epsilon_{S}^2} \\
\pr{\widehat{\corr}_{U} (h) < \corr(h) - \epsilon_{U}} \leq e^{-2n \epsilon_{U}^2} 
\end{eqnarray*}

Let $\delta$ denote the probability we allow for failure. 
If we set $\epsilon_{S} = \sqrt{ \frac{\ln(2p/\delta)}{2m} }$ and $\epsilon_{U} = \sqrt{ \frac{\ln(2p/\delta)}{2n} }$, 
we are guaranteed that \emph{all} the $2p$ inequalities hold concurrently with probability at least $1-\delta$.

We thus set the correlation bound to:
\[
b_i := \widehat{\corr}_{S} (h_i) - \epsilon_{S} - \epsilon_{U}
\]
and have with probability $\geq 1-\delta$ that $\vb$ is a good correlation vector, i.e. $\widehat{\corr}_{U} (h_i) \geq b_i \quad\forall i \in [p]$.


\section{Interpretation and Discussion}
\label{sec:discthmpred}

Given $\sigma$, we partition the examples $\vx$ into three subsets, depending on the value of the ensemble prediction: 
the {\bf hedged set} $ H (\sigma) := \left\{ \vx :  |\vx^\top \sigma|<1 \right\} $, 
the {\bf clipped set} $C (\sigma) := \left\{ \vx :  |\vx^\top \sigma|>1 \right\} $, 
and the {\bf borderline set} $B (\sigma) := \left\{ \vx :  |\vx^\top \sigma| = 1 \right\} $.
Using these sets, we now give some intuition regarding the optimal
choice of $\vg$ and $\vz$ given in~(\ref{eqn:opt-strats}), for some fixed $\sigma$.

Consider first examples $\vx_i$ in $H(\sigma)$. 
Here the optimal $g_i$ is
to predict with the ensemble prediction $\vx_i^\top \sigma$, a number in $(-1, 1)$. 
Making such an intermediate prediction
might seem to be a type of calibration, but this view is misleading. The
optimal strategy for the adversary in this case is to set $z_i = 0$, 
equivalent to predicting $\pm 1$ with probability $1/2$ each. 
The reason that the learner hedges is because if $g_i < \vx_i^\top \sigma$, 
the adversary would respond with $z_i = 1$
and with $z_i = -1$ if $g_i > \vx_i^\top \sigma$. 
In either case, the loss of the predictor would increase. 
In other words, our ultimate rationale for hedging is not calibration, 
but rather ``defensive forecasting'' in the spirit of \cite{VTS05}.

Next we consider the clipped set $\vx_j \in C(\sigma)$. In this case, the adversary's optimal strategy
is to predict deterministically, and so the learner matches the adversary here. 
It is interesting to note that with all else held equal, increasing the margin 
$\abs{\vx_j^\top \sigma}$ beyond 1 is suboptimal for the learner. 
Qualitatively, the reason is that while $\vx_j^\top \sigma$
continues to increase, the prediction for the learner is clipped, 
and so the value for the learner does not increase with the ensemble prediction.

\subsection{Subgradient Conditions}
For another perspective on the result of Theorem~\ref{thm:gamesolngen}, 
consider the subdifferential set of the slack function $\gamma$ at an arbitrary weighting $\sigma$:
\begin{align}
\label{subdiffgammasig}
\partial \gamma (\sigma) 
= \left\{ \frac{1}{n} \lrp{ \sum_{\vx_j \in C(\sigma)} \vx_j \sgn(\vx_j^\top \sigma) 
+ \sum_{\vx_j \in B(\sigma)} c_j \vx_j \sgn(\vx_j^\top \sigma) } - \vb
, \quad \forall c_j \in [0,1]  \right\}
\end{align}
Note that the hedged set plays no role in $\partial \gamma (\sigma)$.
Since the slack function $\gamma(\cdot)$ is convex (Lemma \ref{lem:helperthreshcvx}), 
the sub-differential set \eqref{subdiffgammasig} at any optimal weighting $\sigma^*$ contains $\vec{0}$, 
i.e.,
\begin{align}
\label{gammasgexact}
\exists c_j \in [0,1] \qquad s.t. \qquad 
n \vb - \sum_{j : \vx_{j}^\top \sigma^* > 1} \vx_{j} + \sum_{j : \vx_{j}^\top \sigma^* < -1} \vx_{j} 
&= \sum_{j : \abs{\vx_{j}^\top \sigma^*} = 1} c_j \vx_{j} \sgn(\vx_{j}^\top \sigma^*)
\end{align}
The geometric interpretation of this equation is given in Figure~\ref{fig:optsigma}. 
The optimal weighting $\sigma^*$ partitions the examples into five sets: 
hedged, positive borderline and positive clipped, and negative borderline and negative clipped. 
Taking the difference between the sum of the positive clipped and the sum of the negative
clipped examples gives a vector that is approximately $\vb$. 
By adding a weighted sum of the borderline examples, $\vb$ can be obtained exactly.

\begin{figure}
\label{fig:optsigma}
\centering
\includegraphics[width=0.6\textwidth]{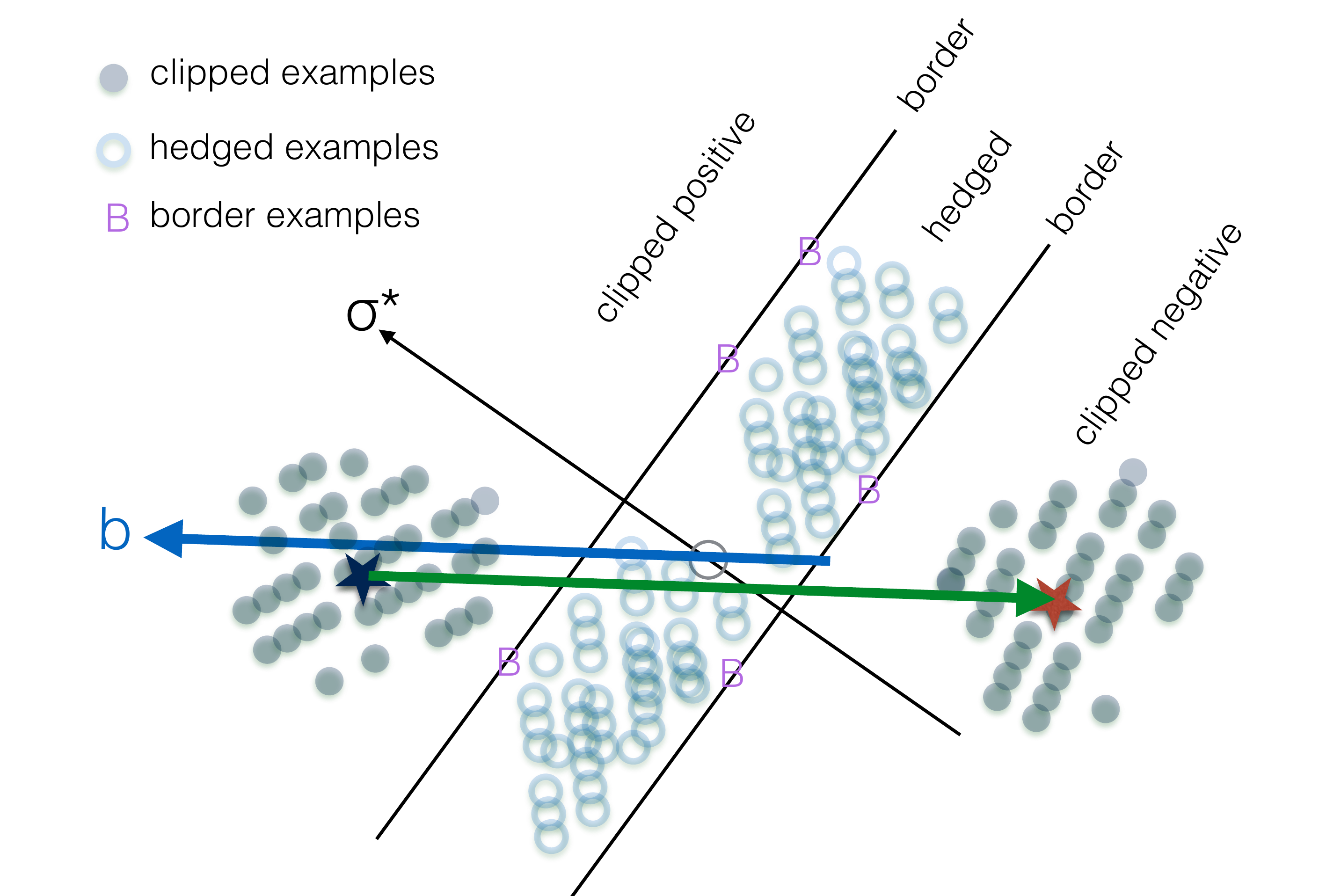}
\caption{\small An illustration of the optimal $\sigma^* \geq 0^p$. 
The vector $n \vb$ is the difference between the sums of two categories of clipped examples: 
those with high ensemble prediction ($\vx^\top \sigma^* > 1$) and low prediction ($< -1$).
The effect of $B (\sigma^*)$ is neglected for simplicity. 
}
\end{figure}

\subsection{Beating ERM Without Clipping}
We now make some brief observations about the minimax solution.

First, note that no $\sigma$ such that $\vnorm{\sigma}_1 < 1$ can be optimal, 
because in such a case $-\gamma(\sigma) > -\gamma \lrp{\frac{\sigma}{\vnorm{\sigma}_1}}$; 
therefore, $\vnorm{\sigma^*}_1 \geq 1$.

Next, suppose we do not know the matrix $\vF$. 
Then $\|\sigma^*\|_1=1$. This can be shown by proving the contrapositive. 
Assuming the negation $1 < \|\sigma^*\|_1 := a$, there exists a 
vector $\vx \in [-1,1]^p$ such that $\vx^\top \sigma^* = \|\sigma^*\|_1 > 1$. 
If each of the columns of $\vF$ is equal to $\vx$, 
then by definition of the slack function, $-\gamma \lrp{\frac{\sigma^*}{a}} > -\gamma(\sigma^*)$, 
so $\sigma^*$ cannot be optimal. 

In other words, if we want to protect ourselves against the worst case $\vF$, then
we have to set $\|\sigma\|_1=1$ so as to ensure that $C(\sigma)$ is empty. 
In this case, the slack function simplifies to $\gamma (\sigma) = - \vb^\top \sigma$, 
over the probability simplex. 
Minimizing this is achieved by setting $\sigma_i$ to be $1$ at $\displaystyle \argmax_{i \in [p]} b_i$ and zero elsewhere. 
So as might be expected, in the case that $\vF$ is unknown, the
optimal strategy is simply to use the classifier with the best error guarantee. 

This is true because $C(\sigma^*)$ is empty, 
and the set of all $\sigma$ such that $C(\sigma^*)$ is empty is of wider interest. 
We dub it the \emph{Zero Box Region}: $\zbr = \left\{ \sigma : C(\sigma) = \emptyset \right\}$.
Another clean characterization of the $\zbr$ can be made by using a 
duality argument similar to that used to prove Theorem \ref{thm:gamesolngen}. 

\begin{thm}
\label{thm:zbrunconstr}
The best weighting in $\zbr$ satisfies 
$\displaystyle \max_{\substack{ \abs{\vF^\top \sigma} \leq \mathbf{1}^n , \\ \sigma \geq 0^p }} \; \vb^\top \sigma 
= \max_{\vg \in [-1,1]^n} \;\min_{\frac{1}{n} \vF \vz \geq \vb } \;\; \frac{1}{n} \vz^\top \vg  $\;. 
In particular, the optimal $\sigma^* \in \zbr$ if and only if 
the hypercube constraint $\vz \in [-1,1]^{n}$ is superfluous, i.e. when 
$\displaystyle V = \min_{ \frac{1}{n} \vF \vz \geq \vb } \;\max_{\vg \in [-1,1]^n} \;\; \frac{1}{n} \vz^\top \vg $\;.
\end{thm}
The $\zbr$ is where the optimal strategy is always to hedge 
and never to incorporate any clipping. 
Consider a situation in which the solution is in $\zbr$, $\sigma^*=\vone^p$, 
and all of the predictions are binary: $\vF \in \{-1,+1\}^{p \times n}$.
This is an ideal case for our method; instead of the baseline value $\max_i b_i$
obtained when $\vF$ is unknown, we get a superior value of $\sum_i b_i$. 

In fact, we referred to such a case in the introduction, 
and we present a formal version here. 
Take $p$ to be odd and suppose that $n=p$. 
Then set $\vF$ to be a matrix where each row (classifier) and each column (example)
contains $(p+1)/2$ entries equal to $+1$ and $(p-1)/2$ entries equal to $-1$.
\footnote{For instance, by setting $F_{ij} = 1$ if $(i+j) \text{ is even}$, and $-1$ otherwise.} 
Finally choose an arbitrary subset of the columns (to have true label $-1$), 
and invert all their entries. 

In this setup, all classifiers (rows) have the same error: $\vb= \frac{1}{p} \vone^p$. 
The optimal weight vector in this case is $\sigma^*=\vone^p$, 
the solution is in $\zbr$ because $\abs{\vx^\top \sigma^*} = 1 \;\forall \vx$, and the minimax value is $V=1$, 
which corresponds to zero error. 
Any single rule has an error of $\frac{1}{2} - \frac{1}{p}$, 
so using $\vF$ with $p$ classifiers has led to a $p$-fold improvement over random guessing!

Of course, this particular case is extremal in some ways; in order to be in $\zbr$, 
there must be many cancellations in $\vF^\top \sigma^*$. 
This echoes the common heuristic belief that, when combining an ensemble of classifiers, 
we want the classifiers to be ``diverse" (e.g. \cite{K03}). 
The above example in fact has the maximal average disagreement between pairs of classifiers for a fixed $p$. 
Similar results hold if $\vF$ is constructed using independent random draws. 

So our formulation recovers ERM without knowledge of $\vF$, 
and can recover an (unweighted) majority vote in cases where this provides dramatic performance improvements. 
The real algorithmic benefit of our unified formulation is in automatically interpolating between these extremes. 

To illustrate, suppose $\vF$ is given in Table~\ref{tbl:Example}, in which there are six classifiers partitioned into two blocs, 
and six equiprobable test examples. 
Here, it can be seen that the true labeling must be $+$ on all examples.

\begin{table}[H]
\begin{center}
\begin{tabular}{||c||c|c|c||c|c|c|}
    \hline &\multicolumn{3}{c||}{A classifiers} & \multicolumn{3}{c|}{B classifiers} \\
    \hline \hline
    &$h_1$ &$h_2$	&$h_3$	&$h_4$ &$h_5$ &$h_6$ 	\\
    \hline \hline
    $x_1$ &- 	&	+	&	+	& + & + & + 	\\ \hline
    $x_2$ &- 	&	+	&	+	& + & + & + 	\\ \hline 
    $x_3$ &+ 	&	-	&	+	& + & + & +     \\ \hline
    $x_4$ &+ 	&	-	&	+	& + & + & +     \\ \hline 
    $x_5$ &+ 	&	+	&	-	& + & + & +	\\ \hline 
    $x_6$ &+ 	&	+	&	-	& -  & -  & -	\\ \hline \thickhline
    $\vb$  &1/3 	&1/3	&1/3	&2/3  &2/3  &2/3	\\ \thickhline
\end{tabular}
\end{center}
\vspace{-2mm}
\caption{\label{tbl:Example}Example with two classifier blocs.}
\end{table}

In this situation, the best single rule errs on $x_6$, as does an (unweighted majority) vote over the six classifiers, 
and even a vote over just the better-performing ``B" classifiers. 
But a vote over the ``A'' rules makes no errors, and our algorithm recovers it 
with a weighting of $\sigma^* = (1;1;1;0;0;0)$.

\subsection{Approximate Learning}

Another consequence of our formulation 
is that predictions of the form $\vg (\sigma)$ 
are closely related to dual optima and the slack function. 
Indeed, by definition of $\vg (\sigma)$, the slack function value 
$ - \gamma (\sigma) = \vb^\top \sigma - \frac{1}{n} \vnorm{\vF^\top \sigma - \vg (\sigma)}_1 
\leq \max_{\sigma' \geq 0^p} \left[ \vb^\top \sigma' - \frac{1}{n} \vnorm{\vF^\top \sigma' - \vg (\sigma)}_1 \right]$, 
which is simply the dual problem (Lemma \ref{lem:gamegeng}) of the worst-case correlation suffered by $\vg (\sigma)$: 
$\displaystyle \quad \min_{\substack{ \vz \in [-1,1]^n , \\ \frac{1}{n} \vF \vz \geq \vb }} \;\frac{1}{n} \vz^\top [\vg (\sigma)]$. 
We now state this formally.

\begin{obs}
\label{obs:slacksubopt}
For any weight vector $\sigma \geq 0^p$, 
the worst-case correlation after playing $\vg (\sigma)$ is bounded by 
$$ \quad \min_{\substack{ \vz \in [-1,1]^n , \\ \frac{1}{n} \vF \vz \geq \vb }} \;\frac{1}{n} \vz^\top [\vg (\sigma)] 
\geq - \gamma (\sigma) $$
\end{obs}

Observation \ref{obs:slacksubopt} shows that convergence guarantees for optimizing the slack function 
directly imply error guarantees on predictors of the form $\vg (\sigma)$, 
i.e. prediction rules of the form in Fig. \ref{fig:optstrats}.

\subsection{Independent Label Noise}

An interesting variation on the game is to limit the adversary to 
$z_i \in [-\alpha_i,\alpha_i]^n$ for some $\vec{\alpha} = (\alpha_1 ; \dots ; \alpha_n) \in [0,1)^n$. 
This corresponds to assuming a level $1 - \alpha_i$ of independent label noise on example $i$: 
the adversary is not allowed to set the label deterministically, 
but is forced to flip example $i$'s label independently 
with probability $\frac{1}{2}(1 - \alpha_i)$. 

Solving the game in this case gives the result (proof in appendices) 
that if we know some of the ensemble's errors to be through random noise, 
then we can find a weight vector $\sigma$ that would give us better performance than 
without such information. 
\begin{prop}[Independent Label Noise]
\label{prop:labnoise}
\begin{align*}
\max_{\vg \in [-1,1]^n} \;\;\min_{\substack{ - \valpha \leq \vz \leq \valpha , \\ 
\frac{1}{n} \vF \vz \geq \vb }} \;\; \frac{1}{n} \vz{^\top} \vg 
= \max_{\sigma \geq 0^p } \;\;\vb^\top \sigma - \frac{1}{n} \sum_{j=1}^n \alpha_j \left[ \abs{x_{j}^\top \sigma} - 1 \right]_{+}
> \max_{\sigma \geq 0^p } \;[- \gamma (\sigma)]
= V 
\end{align*}
\end{prop}
Our prediction tends to clip -- predict with the majority vote -- more on examples with more known random noise, 
because it gains in minimax correlation by doing so. 
This mimics the Bayes-optimal classifier, which is always a majority vote. 

Indeed, this statement's generalization to the asymmetric-noise case 
can be understood with precisely the same intuition. 
The sign of the majority vote affects the clipping penalty in the same way:
\begin{prop}[Asymmetric Label Noise]
\label{prop:asymlabnoise}
For some $\vl, \vu \geq \vzero^n$,
\begin{align*}
\max_{\vg \in [-1,1]^n} \;\;\min_{\substack{ -\vl \leq \vz \leq \vu , \\ 
\frac{1}{n} \vF \vz \geq \vb }} \;\; \frac{1}{n} \vz{^\top} \vg 
= \max_{\sigma \geq 0^p } \;\;\vb^\top \sigma - \frac{1}{n} \sum_{j=1}^n \lrp{ u_j \left[ x_{j}^\top \sigma - 1 \right]_{+} + l_j \left[ - x_{j}^\top \sigma - 1 \right]_{+} }
> V 
\end{align*}
\end{prop}


\section{Computational Issues}
\label{sec:alg}

The learning algorithm we presented has two steps. 
The first is efficient and straightforward: 
$\vb$ is calculated by simply averaging over training examples to produce $\widehat{\corr}_{S} (h_i)$. 

So our ability to produce $\vg^*$ 
is dependent on our ability to find the optimal weighting $\sigma^*$ 
by minimizing the slack function $\gamma (\sigma)$ over $\sigma \geq 0^p$. 
Note that typically $p \ll n$, and so it is a great computational benefit in this case 
that the optimization is in the dual. 

We discuss two approaches to minimizing the slack function. 
The most straightforward approach is to treat the problem as a linear
programming problem and use an LP solver. The main problem with this
approach is that it requires storing all of the examples in memory. As
unlabeled examples are typically much more plentiful than labeled
examples, this approach could be practically infeasible without further modification.

A different approach that exploits the structure of the equilibrium uses stochastic gradient descent (SGD). The fact
that the slack function is convex guarantees that this approach will converge to the global minimum. 
The convergence rate might be suboptimal, particularly near the intersections of hyperplanes 
in the piecewise-linear slack function surface. But the fact that SGD is a constant-memory algorithm is very attractive. 

Indeed, the arsenal of stochastic convex optimization methods comes into play theoretically and practically. 
The slack function is a sum of i.i.d. random variables, 
and has a natural limiting object $\evp{\left[ \abs{\vx^\top \sigma} - 1 \right]_{+}}{\vx \sim \cD} - \vb^\top \sigma$ 
amenable to standard optimization techniques.

\section{Related Work}
\label{sec:relwork}
Our duality-based formulation would incorporate constraints far beyond the linear ones we have imposed so far, 
since all our results hold essentially without change in a general convex analysis context. 
Possible extensions in this vein include other loss functions as in multiclass and abstaining settings, specialist experts, 
and more discussed in the next section.


Weighted majority votes are a nontrivial ensemble aggregation method 
that has received focused theoretical attention for classification. 
Of particular note is the literature on boosting for forming ensembles, 
in which the classic work of \cite{SFBL98} shows general bounds 
on the error of a weighted majority vote $\epsilon_{\wmv} (\sighat)$ under any distribution $\sighat$,
based purely on the distribution of a version of the margin on labeled data. 

Our worst-case formulation here gives direct bounds on (expected) test error $\epsilon_{\wmv} (\sighat)$ as well, 
since in our transductive setting, these are equivalent to lower bounds on the slack function value by Observation \ref{obs:slacksubopt}. 
As we have abstracted away the labeled data information into $\vb$, our results depend only on $\vb$ and 
the distribution of margins $\abs{\vx^\top \sigma}$ among the unlabeled data. 
Interestingly, \cite{AUL09} take a related approach to prove bounds on $\epsilon_{\wmv} (\sighat)$  
in a transductive setting, as a function of the average ensemble error $\vb^\top \sighat$ and the test data margin distribution; 
but their budgeting is looser and purely deals with majority votes, 
in contrast to our $\vg$ in a hypercube. 
The transductive setting has general benefits for averaging-based bounds also (\cite{BL03}).

One class of philosophically related methods to ours uses moments of labeled data 
in the statistical learning setting to find a minimax optimal classifier; 
notably among linear separators (\cite{LGBJ01}) and conditional label distributions under log loss (\cite{LZ14}). 
Our formulation instead uses only one such moment and focuses on unlabeled data, and is 
thereby more efficiently able to handle 
a rich class of dependence structure among classifier predictions, not just low-order moments. 

There is also a long tradition of analyzing worst-case binary prediction of online sequences, 
from which we highlight \cite{FMG92}, 
which shows universal optimality for bit prediction of a piecewise linear function similar to Fig. \ref{fig:optstrats}. 
The work of \cite{CBFHHSW93} demonstrated this to result in optimal prediction error 
in the experts setting as well, and similar results have been shown in related settings (\cite{V90, AP13}). 

Our emphasis on the benefit of considering global effects (our transductive setting) even when data are i.i.d. 
is in the spirit of the idea of shrinkage, well known in statistical literature since at least the James-Stein estimator (\cite{EM77}).


\section{Conclusions and Open Problems}
\label{sec:openproblems}

In this paper we have given a new method of utilizing unlabeled examples
when combining an ensemble of classifiers. We showed that in some cases,
the performance of the combined classifiers is guaranteed to be much
better than that of any of the individual rules.

We have also shown that the optimal solution is characterized by a convex
function we call the slack function. Minimizing this slack function is
computationally tractable, and can potentially be solved in a streaming
model using stochastic gradient descent. The analysis introduced an 
ensemble prediction $\vx^\top \sigma$ similar to the margin used in support vector machines. 
Curiously, the goal of the optimization problem is to minimize, rather than
maximize, the number of examples with large margin.

Directions we are considering for future research include:
\begin{itemize}
\item Is there an algorithm that combines the convergence rate of the
  linear programming approach with the small memory requirements of
  SGD?
\item In problems with high Bayes error, what is the best way to
  leverage the generalized algorithm which limits the adversary to
  a sub-interval of $[-1,1]$?
\item Can the algorithm and its analysis be extended to infinite
  concept classes, and under what conditions can this be done efficiently?
\item Allowing the classifiers to abstain can greatly increase the
  representational ability of the combination. Is there a systematic
  way to build and combine such ``specialist'' classifiers?
\end{itemize}

\section*{Acknowledgements}

The authors are grateful to the National Science Foundation for support under grant IIS-1162581.



\newpage
\bibliography{gameConf-colt-final}{}

\newpage
\appendix

\section{Proof of Theorem \ref{thm:gamesolngen}}
\label{sec:pfgame1}

The core duality argument in the proofs is encapsulated in  
an independently useful supporting lemma describing the adversary's response to a given $\vg$.
\begin{lem}
\label{lem:gamegeng}
For any $\vg \in [-1,1]^n$,
\begin{align*}
\min_{\substack{ \vz \in [-1,1]^n , \\ \frac{1}{n} \vF \vz \geq \vb }} \;\;\frac{1}{n} \vz^\top \vg 
\;=\; \max_{\sigma \geq 0^p} \left[ \vb^\top \sigma - \frac{1}{n} \vnorm{\vF^\top \sigma - \vg}_1 \right]
\end{align*}
\end{lem}
\begin{proof}
We have
\begin{align}
\label{eq:primallagrange}
\displaystyle \min_{\substack{ \vz \in [-1,1]^n , \\ \vF \vz \geq n \vb }} \;\; \frac{1}{n} \vz^\top \vg 
&= \frac{1}{n} \min_{\vz \in [-1,1]^n} \; \max_{\sigma \geq 0^p} \;  \lrb{ \vz^\top \vg - \sigma^\top (\vF \vz - n \vb) } \\
\label{eq:duallagrange}
&\stackrel{(a)}{=} \frac{1}{n} \max_{\sigma \geq 0^p} \; \min_{\vz \in [-1,1]^n} \; \lrb{ \vz^\top (\vg - \vF^\top \sigma) + n \vb^\top \sigma } \\
\label{eq:gameinnerdual}
&= \frac{1}{n} \max_{\sigma \geq 0^p} \; \lrb{ - \vnorm{\vg - \vF^\top \sigma}_1 + n \vb^\top \sigma } 
= \max_{\sigma \geq 0^p} \left[ \vb^\top \sigma - \frac{1}{n} \vnorm{\vF^\top \sigma - \vg}_1 \right]
\end{align}
where $(a)$ is by the minimax theorem. 
\end{proof}

Now $\vg^*$ and $V$ can be derived.

\begin{lem}
\label{lem:game1gopt}
If $\sigma^*$ is defined as in Theorem \ref{thm:gamesolngen}, then for every $i \in [n]$,
\begin{align*}
g_i^* = 
\begin{cases} [\vF^\top \sigma^*]_{i} & \abs{[\vF^\top \sigma^*]_{i}} < 1 \\ 
\sgn([\vF^\top \sigma^*]_{i}) & otherwise \end{cases}
\end{align*}
Also, the value of the game \eqref{game1eq} is $V$, as defined in Theorem \ref{thm:gamesolngen}.
\end{lem}
\begin{proof}[Proof of Lemma \ref{lem:game1gopt}]
From Lemma \ref{lem:gamegeng}, 
the primal game \eqref{game1eq} is equivalent to 
\begin{align}
\label{gamehalfdual}
\max_{\substack{ \vg \in [-1,1]^n, \\ \sigma \geq 0^p }} \left[ \vb^\top \sigma - \frac{1}{n} \vnorm{\vg - \vF^\top \sigma}_1 \right]
\end{align}
From \eqref{gamehalfdual}, 
it is clear that given a setting of $\sigma$, 
the $\vg^* (\sigma)$ that maximizes \eqref{gamehalfdual} 
is also the one that minimizes $\vnorm{\vg - \vF^\top \sigma}_1$ under the hypercube constraint:
\begin{align*}
g_i^* (\sigma) = \begin{cases} [ \vF^\top \sigma ]_{i}, & \abs{[\vF^\top \sigma ]_{i}} < 1 \\ 
\sgn([\vF^\top \sigma]_{i}), & otherwise \end{cases}
\end{align*}
The optimum $\sigma$ here is therefore
\begin{align*}
\sigma^* &= \argmax_{\sigma \geq 0^p} \left[ \vb^\top \sigma - \frac{1}{n} \vnorm{\vg^* (\sigma) - \vF^\top \sigma}_1 \right] \\
&= \argmax_{\sigma \geq 0^p} \left[ \vb^\top \sigma - \frac{1}{n} \sum_{j=1}^n \left[ \abs{\vx_{j}^\top \sigma} - 1 \right]_{+} \right] 
= \argmin_{\sigma \geq 0^p} [\gamma (\sigma)]
\end{align*}
which finishes the proof.
\end{proof}

\subsection{Derivation of $\vz^*$}
\label{sec:optStrategies}

We can now derive $\vz^*$ from Lagrange complementary slackness conditions.

\begin{lem}
\label{lem:game1zopt}
If $\sigma^*$ is defined as in Theorem \ref{thm:gamesolngen}, then for every $i \in [n]$,
\begin{align*}
z_i^* = 
\begin{cases} 
0, & \abs{[\vF^\top \sigma^*]_{i}} < 1 \\ 
\sgn([\vF^\top \sigma^*]_{i}), & \abs{[\vF^\top \sigma^*]_{i}} > 1 
\end{cases}
\end{align*}
\end{lem}

\begin{proof}
We first rewrite the game slightly to make 
complementary slackness manipulations more transparent. 

Define $c = \mathbf{1}^{2n}$, 
$A = [\vF, -\vF; - I_{2n}] \in [-1, 1]^{(p + 2n) \times 2n} $, 
and $B = [n \vb; -\vone^{2n}] \in \mathbb{R}^{p + 2n}$. 
This will allow us to reparametrize the problem in terms of $\zeta = ([z_1]_{+}, \dots, [z_n]_{+}, [z_1]_{-}, \dots, [z_n]_{-})^\top$. 

Now apply the minimax theorem (\cite{CBL06}, Theorem 7.1) 
to \eqref{game1eq} to yield the minimax dual game: 
\begin{align}
\label{mmxdualbase}
\displaystyle \min_{\substack{ \vz \in [-1,1]^n , \\ \frac{1}{n} \vF \vz \geq \vb }} \max_{\vg \in [-1,1]^n} \;\; \frac{1}{n} \vz^\top \vg
= \min_{\substack{ \vz \in [-1,1]^n , \\ \vF \vz \geq n \vb }} \;\; \frac{1}{n} \vnorm{\vz }_1 \;\;
\end{align}

With the above definitions, 
\eqref{mmxdualbase} becomes 
\begin{align}
\label{mmxdualnew}
\min_{A \zeta \geq B , \zeta \geq 0} \frac{1}{n} c^\top \zeta
\end{align}

This is clearly a linear program (LP); its dual program, equal to it by strong LP duality 
(since we assume a feasible solution exists; \cite{V96}) is 
\begin{align}
\label{mmxdualsdualinterm}
\max_{\lambda \geq 0, c \geq A^\top \lambda} \frac{1}{n} B^\top \lambda
\end{align}
Denote the solutions to \eqref{mmxdualnew} and \eqref{mmxdualsdualinterm} as $\zeta^*$ and $\lambda^*$.

By Lemma \ref{lem:game1gopt} and the discussion leading up to it, we already know $\sigma^*$, 
and therefore only need establish the dependence of $\vz^*$ on $\sigma^*$.
Applying LP complementary slackness (\cite{V96}, Thm. 5.3) to \eqref{mmxdualnew} and \eqref{mmxdualsdualinterm}, 
we get that for all $j \in [2n]$,
\begin{align}
\label{compslack1}
[c - A^\top \lambda^*]_{j} > 0 \implies \zeta_{j}^* = 0
\end{align}
and
\begin{align}
\label{compslack2}
\lambda_{j}^* > 0 \implies [A \zeta^* - B]_{j} = 0
\end{align}

First we examine \eqref{compslack1}. 
The condition $[c - A^\top \lambda^*]_{j} > 0$ can be rewritten as 
\begin{align*}
c_j + [ \lambda_{-}^* ]_{j} - [\vF^\top \sigma^*; - \vF^\top \sigma^*]_{j} > 0
\end{align*}
For any example $i \leq n$, if $\abs{[\vF^\top \sigma^*]_i} < 1$, 
then $[c - A^\top \lambda^*]_{i} \geq c_{i} - [\vF^\top \sigma^*; - \vF^\top \sigma^*]_{i} > 0$ 
(since $\lambda_{-}^* \geq 0$).
Similarly, 
$c_{i+n} - [\vF^\top \sigma^*; - \vF^\top \sigma^*]_{i+n} > 0$. 
By \eqref{compslack1}, this means $\zeta_{i}^* = \zeta_{i+n}^* = 0$, 
which implies $z_{i}^* = 0$ by definition of $\zeta^*$. 
So we have shown that 
$\abs{[\vF^\top \sigma^*]_i} < 1 \implies z_{i}^* = 0$. 

First we examine \eqref{compslack1}. 
The condition $[c - A^\top \lambda^*]_{j} > 0$ can be rewritten as 
\begin{align*}
c_j + [ \lambda_{-}^* ]_{j} - [\vF^\top \sigma^*; - \vF^\top \sigma^*]_{j} > 0
\end{align*}
For any example $i \leq n$, if $\abs{[\vF^\top \sigma^*]_i} < 1$, 
then $[c - A^\top \lambda^*]_{i} \geq c_{i} - [\vF^\top \sigma^*; - \vF^\top \sigma^*]_{i} > 0$ 
(since $\lambda_{-}^* \geq 0$).
Similarly, 
$c_{i+n} - [\vF^\top \sigma^*; - \vF^\top \sigma^*]_{i+n} > 0$. 
By \eqref{compslack1}, this means $\zeta_{i}^* = \zeta_{i+n}^* = 0$, 
which implies $z_{i}^* = 0$ by definition of $\zeta^*$. 
So we have shown that 
$\abs{[\vF^\top \sigma^*]_i} < 1 \implies z_{i}^* = 0$. 

It remains only to prove that 
$\abs{[\vF^\top \sigma^*]_i} > 1 \implies z_{i}^* = \sgn([\vF^\top \sigma^*]_i)$ for any $i \in [n]$. 
We first show this is true for $[\vF^\top \sigma^*]_i > 1$.
In this case, from the constraints we need $c \geq A^\top \lambda^*$, 
in particular that 
\begin{align}
\label{dualnonslackconstr}
0 \leq [c - A^\top \lambda^*]_{i} = [ \lambda_{-}^* ]_{i} + (c_i - [\vF^\top \sigma^*]_{i})
\end{align}
By assumption, $c_i - [\vF^\top \sigma^*]_{i} = 1 - [\vF^\top \sigma^*]_{i} < 0$. 
Combined with \eqref{dualnonslackconstr}, 
this means we must have $[ \lambda_{-}^* ]_{i} > 0$, i.e. $\lambda_{i+p}^* > 0$. 
From \eqref{compslack2}, 
this means that 
$$[A \zeta^* - B]_{i+p} = 0 \iff \zeta_i^* = 1 $$
Meanwhile, 
\begin{align*}
[c - A^\top \lambda^*]_{i+n} = [ \lambda_{-}^* ]_{i+n} + (c_{i+n} - [- \vF^\top \sigma^*]_{i})
\geq c_{i+n} + [\vF^\top \sigma^*]_{i} > 0
\end{align*}
so from \eqref{compslack1}, $\zeta_{i+n}^* = 0$. 
Since $\zeta_i^* = 1$, this implies $z_i^* = 1 = \sgn([\vF^\top \sigma^*]_i)$, as desired.

This concludes the proof for examples $i$ such that $[\vF^\top \sigma^*]_i > 1$. 
The situation when $[\vF^\top \sigma^*]_i < -1$ is similar, 
but the roles of the $i^{th}$ and $(i+n)^{th}$ coordinates are reversed 
from \eqref{dualnonslackconstr} onwards in the above proof.
\end{proof}

By further inspection of the subgradient conditions described in the body of the paper, 
one can readily show the following result, which complements Theorem \ref{thm:gamesolngen}.
\begin{cor}
\label{cor:zoptfullpred}
For examples $j$ such that $\abs{\vx_j^\top \sigma^*} = 1$, 
$$ z_j^* = c_j \sgn(\vx_j^\top \sigma^*) $$
where $c_j \in [0,1]$ are as defined in \eqref{gammasgexact}.
\end{cor}


\section{Miscellaneous Proofs}
\label{sec:miscpfs}

\begin{lem}
\label{lem:helperthreshcvx}
The function $\gamma(\sigma)$ is convex in $\sigma$.
\end{lem}
\begin{proof}
To prove Part $1$, note that for each $j$, 
the term $\left[ \abs{\vx_{j}^\top \sigma} - 1 \right]_{+}$ is convex in $\sigma$. 
Therefore, the average of $n$ terms is convex. 
As the term $- \vb^\top \sigma$ is linear, the whole
expression $\gamma (\sigma)$ is convex.
(This is a special case of the Lagrangian dual function always being concave in the dual variables.)
\end{proof}

\begin{proof}[Proof of Theorem \ref{thm:zbrunconstr}]
Let $A = [\vF, -\vF] \in \RR^{p \times 2n}$ and $c = \mathbf{1}^{2n}$. 
Then the first assertion follows by LP duality: 
\begin{align*}
\max_{\substack{ \abs{\vF^\top \sigma} \leq \mathbf{1}^n , \\ \sigma \geq 0^p }} \; \vb^\top \sigma 
&= \max_{\substack{ A^\top \sigma \leq c , \\ \sigma \geq 0^p }} \; \vb^\top \sigma 
\stackrel{(a)}{=} \min_{\substack{ A \zeta \geq \vb , \\ \zeta \geq 0^{2n} }} \; c^\top \zeta 
= \min_{\substack{ \vF (\zeta^{\alpha} - \zeta^{\beta}) \geq \vb , \\ \zeta^{\alpha} , \zeta^{\beta} \geq 0^{n} }} \; 
\left[ \vnorm{\zeta^{\alpha}}_1 + \vnorm{\zeta^{\beta}}_1 \right] 
= \min_{ \vF \vz \geq \vb } \; \vnorm{\vz}_1 \\
&= \min_{ \frac{1}{n} \vF \vz \geq \vb } \; \frac{1}{n} \vnorm{\vz}_1 
= \min_{\frac{1}{n} \vF \vz \geq \vb } \;\max_{\vg \in [-1,1]^n} \;\; \frac{1}{n} \vz^\top \vg
\stackrel{(b)}{=} \max_{\vg \in [-1,1]^n} \;\min_{\frac{1}{n} \vF \vz \geq \vb } \;\; \frac{1}{n} \vz^\top \vg
\end{align*}
where $(a)$ is by strong LP duality and $(b)$ uses the minimax theorem. 

The second assertion follows because 
$\displaystyle V = - \gamma (\sigma^*) \stackrel{(c)}{=} \vb^\top \sigma^* 
\stackrel{(d)}{=} \max_{\vg \in [-1,1]^n} \;\min_{\frac{1}{n} \vF \vz \geq \vb } \;\; \frac{1}{n} \vz^\top \vg$\;, 
where $(c)$ uses the definition of $\zbr$ 
and $(d)$ is due to the first assertion. 
\end{proof}

\begin{proof}[Proof of Prop. \ref{prop:labnoise}]
The derivation here closely follows that of Lemma \ref{lem:gamegeng}, except that 
\eqref{eq:gameinnerdual} now instead becomes 
\begin{align*}
\frac{1}{n} \max_{\sigma \geq 0^p} \; \lrb{ - \alpha \vnorm{\vg - \vF^\top \sigma}_1 + n \vb^\top \sigma } 
\end{align*}
which is equal to the final result.
\end{proof}

The proof of Prop. \ref{prop:asymlabnoise} is exactly analogous to that of Prop. \ref{prop:labnoise}.


\end{document}